\documentclass{amse-new111}

\numberwithin{equation}{section} 

\usepackage{tabularx}
\usepackage{algorithm,algorithmic}
\usepackage{booktabs}
\usepackage{multirow}
\usepackage{color}

\begin{document}

 \PageNum{1}
 \Volume{201x}{Sep.}{x}{x}
 \OnlineTime{August 15, 201x}
 \DOI{0000000000000000}

\abovedisplayskip 6pt plus 2pt minus 2pt \belowdisplayskip 6pt
plus 2pt minus 2pt
\def\vsp{\vspace{1mm}}
\def\th#1{\vspace{1mm}\noindent{\bf #1}\quad}
\def\proof{\vspace{1mm}\noindent{\it Proof}\quad}
\def\no{\nonumber}
\newenvironment{prof}[1][Proof]{\noindent\textit{#1}\quad }
{\hfill $\Box$\vspace{0.7mm}}
\def\q{\quad} \def\qq{\qquad}
\allowdisplaybreaks[4]


\AuthorMark{Su W. Q. et al. }                             

\TitleMark{Differentially Private Precision Matrix Estimation}  

\title{Differentially Private Precision Matrix Estimation        
\footnote{Supported by National Natural Science Foundation of China (Grant Nos. 11571011 and U1811461)}}                  

\author{Wen Qing \uppercase{Su}}             
	{School of mathematics, Northwest University, Xi'an 710127, P. R. China\\
    E-mail\,$:$ suwenqing@stumail.nwu.edu.cn }

\author{Xiao \uppercase{Guo}}             
    {School of mathematics, Northwest University, Xi'an 710127, P. R. China\\
    E-mail\,$:$ guoxiao@stumail.nwu.edu.cn }

\author{Hai \uppercase{zhang*}}             
    {School of mathematics, Northwest University, Xi'an 710127, P. R. china\\
    Faculty of Information Technology, Macau University of Science and Technology, Macau, P. R. China\\
    E-mail\,$:$ zhanghai@nwu.edu.cn\thanks{*The corresponding author}}

\maketitle%

\Abstract{ In this paper, we study the problem of precision matrix estimation when the dataset contains sensitive information. In the differential privacy framework, we develop a differentially private ridge estimator by perturbing the sample covariance matrix. Then we develop a differentially private graphical lasso estimator by using the alternating direction method of multipliers (ADMM) algorithm. The theoretical results and empirical results that show the utility of the proposed methods are also provided.
}      



\section{Introduction} \label{section introduction}

Precision matrix plays a fundamental role in many statistical inference problems. For example, in discriminant analysis, the precision matrix needs to be estimated to compute the classification rules\cite{martin1979multivariate}. In graphical models, the structure exploration of gaussian graphical model is equivalent to recover the support of the precision matrix\cite{lauritzen1996graphical}. Moreover, the precision matrix is useful for a wide range of applications including portfolio optimization, genomics and single processing, among many others. Therefore, it is of great importance to estimate the precision matrix.

Given a $p\times n$ data matrix $X$, a sample of $n$ realizations from a $p$-dimensional Gaussian distribution with zero mean and covariance matrix $\Sigma^*$. A natural way to estimate precision matrix $\Theta^*={\Sigma^*}^{-1}$ is via maximum likelihood approach. Under the Gaussian setting, the negative log-likelihood takes the form
\begin{equation}
    -\log |\Theta| +\mbox{tr}(S \Theta), \label{log-likelihood}
\end{equation}
where $S=XX^{T}/n$ is the sample covariance matrix and $\mbox{tr}(\cdot)$ denotes the trace of the matrix. Then minimizing \eqref{log-likelihood} with respect to $\Theta$ yields the maximum likelihood estimate of the precision matrix. Due to the maximum likelihood estimate performs poorly and beyond compute in high-dimensional setting, the penalized log-likelihood functions and other constrained optimization techniques are used to gain better estimates of the matrix. These alternative methods can be coarsely sorted into two types. One is to seek the sparsity in the precision matrix to explore the structure of the Gaussian graphical model. The graphical lasso \cite{banerjee2008model,yuan2007model,friedman2008sparse} is a popular technique among sparse approaches, which penalizes \eqref{log-likelihood} with the $l_1$ norm to induce sparsity and minimizes the penalized log-likelihood
\begin{equation}
    -\log |\Theta| +\mbox{tr}(S \Theta) +\lambda \| \Theta\| _1, \label{glasso}
\end{equation}
over all positive definite matrices $\Theta$. Here $\| \cdot \| _1$ denotes the sum of the absolute values of the entries, and $\lambda$ is a nonnegative tuning parameter. Other sparse approaches can be seen in Cai et al. \cite{cai2011Aconstrained}, Yuan and Wang \cite{yuan2013Acoordinate} and Liu and Luo \cite{liu2015fast}, among others. Another approach is to construct shrinkage estimates of the matrix, which might be useful when the true matrix is non-sparse. Many results on this topic were developed to estimate the covariance matrix, which includes Ledoit and Wolf \cite{ledoit2004Awellcondition}, Warton \cite{warton2008penalized} and Deng and Tsui \cite{deng2013penalized}. Recently, van Wieringen and Peeters \cite{vanWieringen2016ridge} and Kuismin et al. \cite{kuismin2017precision} proposed a ridge-type estimation for the precision matrix, which penalizes \eqref{log-likelihood} with the squared Frobenius norm and minimizes the penalized log-likelihood 
\begin{equation}
    -\log |\Theta|+\mbox{tr}(S\Theta)+\lambda\|\Theta\| _F^2, \label{gridge}
\end{equation}
over all positive definite matrices $\Theta$. Here $\lambda$ is a nonnegative tuning parameter, and $\|\cdot\|_F$ denotes the square root of the sum of the squares of the entries. The advantage of \eqref{gridge} is that there is an analytical expression for the precision matrix estimation.

All of the aforementioned approaches do not make any distinction between sensitive data and non sensitive data. However in real-world data, a large number of datasets contain personal sensitive information. For example, datasets related to portfolio optimization may contain personal financial information, and datasets related to gene expression may contain personal health information. The availability of large datasets containing sensitive information from individuals has motivated the study of learning algorithms that guarantee the privacy of individuals who contribute to the database. A strict and widely used privacy guarantee is via the concept of differential privacy \cite{dwork2006calibrating}. It requires that datasets differing in only one entry induce similar distribution on the output of a algorithm, which will provides strong provable protection to individuals and prevents them from being recognized by an arbitrarily powerful adversary. Thus, it is significant to study the estimation of the precision matrix under the differential privacy framework.

Recently, Wang et al.\cite{wang2018differentially} studied the differentially private sparse precision matrix estimation problem from the model perspective. In this paper, we study the differentially private precision matrix estimation problem from the algorithm perspective. Specifically, we design differentially private algorithms for the ridge-type estimation model \eqref{gridge} and the graphical lasso problem \eqref{glasso}. We first develop a differentially private algorithm for the ridge-type estimation of the precision matrix. Note that when the tuning parameter is fixed, the solution of model \eqref{gridge} is determined by the eigenvalues and eigenvectors of $S$. Thus we provide the differential privacy by perturbing $S$. We then develop a differentially private algorithm for the graphical lasso. Inspired by the post-processing nature of the differential privacy, if an algorithm can been decomposed into stable and unstable parts, and only at the stable parts the algorithm access the raw data. Then, satisfying the differential privacy in stable parts will make the whole algorithm satisfy the differential privacy. Note that the ADMM algorithm for the graphical lasso \cite{boyd2011distributed} consists of three steps, but only step 2 accesses the raw dataset. Moreover, this step corresponds to a ridge-type penalty log-likelihood estimation problem similar to model \eqref{gridge}. Thus, we might perform a differentially private algorithm in this step and then treat other steps as a post-processing step, which will not increase the privacy risk. The theoretical guarantees on the performance of our privacy preserving precision matrix estimation algorithms are also established. Numerical studies show the utility of our privacy preserving algorithms.


The remainder of this paper is organized as follows. We introduce the  differential privacy and the ADMM algorithm in Section \ref{Background}. In Section \ref{methods}, we develop differentially private algorithms in the context of precision matrix estimation. The theoretical guarantees on the performance of our differentially private algorithms was provided in Section \ref{Theoretical analysis}. We then test our algorithms under simulation and real data in Section \ref{experiments}, and conclude in Section \ref{conclusion}.

\section{Preliminaries} \label{Background}

In this section, we introduce the background of the differential privacy, some related privacy notions, and the ADMM algorithm for the graphical lasso.

\subsection{Differentially privacy}

Differential privacy (DP) is a rigorous and common definition of the privacy protection of data analysis algorithm \cite{dwork2006calibrating, dwork2014the}. Let the space of data be $\mathcal{Y}$ and datasets $Y, Y^\prime \in \mathcal{Y}^n$. We use the Hamming distance, denoted by $d_H(Y,Y^\prime)$, to measure the distance between datasets $Y$ and $Y^\prime$, that is, the number of data points on which they differ. If $d_H(Y,Y^\prime)=1$, we call they are neighboring, meaning that we can get $Y^\prime$ by changing one of the data points in $Y$. 

\begin{definition}[\cite{dwork2014the}]
    A randomized algorithm $\mathcal{M}$ is $(\epsilon, \delta)$-differentially private if for all neighboring datasets $Y$ and $Y^\prime$ and for all measurable set $\mathcal{S} \subseteq Range(\mathcal{M})$ the following holds
    \begin{equation}
        Pr(\mathcal{A}(Y)\in \mathcal{S}) \leq e^{\epsilon} Pr(\mathcal{A}(Y^\prime)\in \mathcal{S})+ \delta.
    \end{equation}
    If $\delta=0$, we say that $\mathcal{M}$ is $\epsilon$-differentially private.
\end{definition}

We observe that if algorithm $\mathcal{M}$ that provides $\epsilon$-differential privacy also provides $(\epsilon,\delta)$-differential privacy. As a result, the $\epsilon$-differential privacy is much stronger. Here $\epsilon,\delta$ are privacy budget parameters, the smaller they are, the stronger the privacy guarantee. From this definition, it is clear that if we arbitrarily change any individual data point in a dataset, the output of the algorithm does not shift too much.

Dwork et al. \cite{dwork2006calibrating} also provide a framework for developing privacy preserving algorithms by adding noise with a designed distribution to the output of the algorithm, which is called sensitivity method. The $l_2$-sensitivity of a function is defined as follows.
\begin{definition}[\cite{dwork2014the}]
    The $l_2$-sensitivity of a function $f: \mathcal{Y}^n \to R^k$ is
    \begin{equation}
        S_2(f) = \max_{\substack{Y,Y^\prime \in \mathcal{Y}^n \\ d_H(Y,Y^\prime)=1}} \| f(Y)-f(Y^\prime)\| _2.
    \end{equation}
\end{definition}

The Gaussian mechanism is a typical sensitivity method that preserves $(\epsilon,\delta)$-differential privacy.

\begin{definition}[\cite{dwork2014the}]
    Given any function $f: \mathcal{Y}^n \to R^k$, the Gaussian mechanism is defined as 
    \begin{equation}
        \mathcal{M}(Y,f(\cdot),\epsilon)=f(Y)+(Y_1,\ldots, Y_k),
    \end{equation}
    where $Y_i$ are $i.i.d.$ random variables drawn from Gaussian distribution $\mathcal{N} (0, \beta^2)$, where $\beta = S_2(f)\sqrt{2\log(1.25/\delta)}/\epsilon$.
\end{definition}

The sensitivity of an algorithm measures the maximum difference in the output of the algorithm under neighboring datasets, which is similar to the stability of the algorithm. An algorithm is stable means that the output of the algorithm will vary little when given two very similar date sets, and it also means that the sensitivity of the algorithm is low.
When the privacy parameters are fixed, the low sensitivity makes the differentially private algorithm add small noise at high probability, which  means high utility. Thus, a stable non-private algorithm is suitable for development into a differentially private algorithm.

Moreover, differential privacy is closed under post-processing.

\begin{proposition}[\cite{dwork2014the}]
    Let $\mathcal{M}$ is an $(\epsilon, \delta)$-differentially private algorithm. Let $\mathcal{A}$ is an arbitrary data-independent mapping. Then $\mathcal{A} \circ \mathcal{M}$ is $(\epsilon, \delta)$-differentially private.
\end{proposition}
That is, a data analyst, without additional knowledge about the private dataset, cannot compute a function of the output of a differentially private algorithm $\mathcal{M}$ and make it less differentially private. 

\subsection{Alternating direction method of multipliers algorithm for the graphical lasso}

The ADMM is an algorithm that is well suited to distributed convex optimization\cite{gabay1976Adual, boyd2011distributed}. The algorithm solves optimization problems by decomposing them into smaller parts, and each part is easier to handle. We might implement differential privacy by handling some of these parts. Now we take the ADMM algorithm to solve the graphical lasso.

The graphical lasso problem \eqref{glasso} can be rewritten as
\begin{equation}
    \begin{split}
        \min_{\Theta,Z}&~-\log |\Theta| +\mbox{tr}(S \Theta) +\lambda \| Z\| _1\\
        \mbox{s.t.}&~\Theta-Z=0,
    \end{split}
\end{equation}
where $\Theta$ and $Z$ are positive definite matrices.

Then we form the augmented Lagrangian function
\begin{equation}
    L(\Theta,Z,V)=-\log |\Theta| +\mbox{tr}(S \Theta) +\lambda \| Z\| _1+V^T(\Theta-Z)+\frac{\rho}{2}\| \Theta-Z \| _F^2,
\end{equation}
where $V\in R^{p\times p}$ is the dual variable or Lagrange multiplier, and $\rho>0$ is a preselected penalty coefficient. Then the ADMM algorithm alternately updates variables by minimizing the augmented Lagrangian function.

Let $U=(1/\rho)V$ be the scaled dual variable. Using the scaled dual variable, the ADMM algorithm for the graphical lasso is
\begin{eqnarray}
    \Theta^{k+1}& = & \mathop{\arg\min}\limits_{\Theta} \;-\log |\Theta| +\mbox{tr}(S \Theta)+\frac{\rho}{2}\| \Theta-Z^k+U^k \| _F^2  \label{theta-update}, \\ 
     Z^{k+1}& = & \mathop{\arg\min}\limits_{Z} \;\lambda \| Z\| _1+\frac{\rho}{2}\| \Theta^{k+1}-Z+U^k \| _F^2 \label{z-update}, \\
     U^{k+1}& = & \; U^k+\Theta^{k+1}-Z^{k+1}. \label{u-update}
\end{eqnarray}

This algorithm can be simplified much further. Both $\Theta$-minimization step \eqref{theta-update} and $Z$-minimization step \eqref{z-update} have analytical expression (see Boyd et al. \cite{boyd2011distributed} for details).

\section{Methods} \label{methods}

In this section, we first propose a differentially private ridge estimation for the precision matrix. Then notice that the ADMM algorithm for the graphical lasso contains a ridge-type estimate step. A differentially private algorithm for the graphical lasso problem is proposed.

\subsection{Differentially private ridge estimation for the precision matrix}

For the ridge estimation of the precision matrix, Kuismin et al.\cite{kuismin2017precision} showed that the maximizing solution of \eqref{gridge} has the following closed form
\begin{equation}
    \widehat{\Theta}=M\Lambda M^T,   \label{rope}
\end{equation}
where $\Lambda$ is a $p\times p$ diagonal matrix whose diagonal elements are $\widehat{\varphi_i}=2/(\varphi_i+\sqrt{\varphi_i^2+8\lambda}),\quad i=1,\ldots,p,$ sorted in ascending order beginning from the upper left corner $\Lambda$, that is, $\Lambda=\mbox{diag}(\widehat{\varphi_1},\ldots,\widehat{\varphi_p})$, $\widehat{\varphi_1}\geq\ldots\geq\widehat{\varphi_p}$. Here $\varphi_i$'s are the eigenvalue of the sample covariance matrix $S$. $M$ is an orthogonal $p\times p$ matrix whose columns correspond to the eigenvectors of $S$, which are ordered in the same order in which the sample covariance matrix eigenvalues appear in the matrix $\Lambda$. Similarly, van Wieringen and Peeters \cite{vanWieringen2016ridge} given a different form of analytical solution to model \eqref{gridge}. Both results showed that the ridge estimation of the precision matrix has high stability, in other words, has low sensitivity. It is clear that the estimator \eqref{rope} is determined by the eigenvalues and eigenvectors of $S$ when $\lambda$ is fixed. Thus, the estimator will preserve differential privacy as long as the orthogonal eigenvalue decomposition of $S$ preserves differential privacy.

\begin{algorithm}[H]
	\renewcommand{\algorithmicrequire}{\textbf{Input:}}
	\renewcommand{\algorithmicensure}{\textbf{Output:}}
	\caption{Differentially private ridge estimation for the precision matrix (DP-RP)}
	\label{alg:1}
	\begin{algorithmic}[1]
		\REQUIRE $p\times n$ matrix $X$, privacy parameters $\epsilon,\delta>0$, tuning parameter $\lambda>0$.
		\ENSURE $\widehat{\Theta}_{priv}$
        \STATE Let $S=\frac{XX^T}{n}$.
        \STATE Set $\beta=\frac{\sqrt{2}}{n}\frac{\sqrt{2\log(1.25/\delta)}}{\epsilon}$, generate a $p\times p$ symmetric noisy matrix $E$ whose entries are $i.i.d.$ drawn from $\mathcal{N}(0,\beta^2)$.
        \STATE $\widetilde{S}\leftarrow S+E$.
        \STATE Compute $\widehat{\Theta}_{priv}$ by \eqref{rope} based on $\widetilde{S}$.
	\end{algorithmic}  
\end{algorithm}

Without loss of generality, the columns of the data matrix $X$ are assumed to be normalized to have $l_2$ norm at most one, that is, $\|x_i\|_2\leq 1$, $i=1,\ldots,n$. We are interested in the sensitivity of function $f(X)=XX^T/n$, 
\begin{equation}
    S_2(f)= \max_{\substack{X,X^\prime \in R^{p\times n}}} \| f(X)-f(X^\prime)\| _F=\| \frac{XX^T}{n}-\frac{X^{\prime}{X^{\prime}}^T}{n} \| _F\leq \frac{\sqrt{2}}{n},
\end{equation}
where the function $f(X)$ may be viewed as an $p^2$-dimensional real vector. We use the Analyze Gauss algorithm \citep{dwork2014analyze} to perturb $S$, which provides $(\epsilon,\delta)$-differential privacy. Then we perform equation \eqref{rope} on the perturbed $S$, which will yield a ridge estimate of the precision matrix that is $(\epsilon,\delta)$-differentially private. See Algorithm \ref{alg:1} for details.
\subsection{Differentially private algorithm for the graphical lasso}

One algorithm for solving the graphical lasso problem is the ADMM algorithm in Section \ref{Background}. The ADMM algorithm for the graphical lasso consists of a $\Theta$-minimization step, a $Z$-minimization step and a scaled dual variable update. Obviously, the $\Theta$-minimization step accesses the original dataset directly, while the other steps only use the previous update results. Moreover, the $\Theta$-minimization step is a ridge-type penalty log-likelihood estimation similar to model \eqref{gridge}.

Differential privacy is closed under post-processing, then solving the $\Theta$-minimization step in the $k$-th iteration by Algorithm \ref{alg:1} will make the $k$-th iteration provides $(\epsilon,\delta)$-differential privacy. Note that when Algorithm \ref{alg:1} is used here, the analytical solution used in step 4 may be different from \eqref{rope} (see \cite{vanWieringen2016ridge, kuismin2017precision, boyd2011distributed}). If the algorithm converges after $K$ iterations, then the algorithm provides $(K\epsilon,K\delta)$-differential privacy follows from simple composition.

However, if the same privacy parameter of Algorithm \ref{alg:1} is used in each iteration, it is actually equivalent to perturbing $S$ with the same noise level in each iteration. Therefore, we can perturb $S$ directly before running the ADMM algorithm, which will avoid the accumulation of privacy loss, that is, reduces privacy cost. We have the following differentially private ADMM algorithm for the graphical lasso.

\begin{algorithm}[H]
	\renewcommand{\algorithmicrequire}{\textbf{Input:}}
	\renewcommand{\algorithmicensure}{\textbf{Output:}}
	\caption{Differentially private ADMM algorithm for the graphical lasso (DP-AGL)}
	\label{alg:2}
	\begin{algorithmic}[1]
		\REQUIRE $p\times n$ matrix $X$, privacy parameters $\epsilon,\delta>0$, tuning parameter $\lambda>0$, and initial values $Z^0,U^0$.
		\ENSURE $\widehat{\Theta}_{priv}$
        \STATE Let $S=\frac{XX^T}{n}$.
        \STATE Set $\beta=\frac{\sqrt{2}}{n}\frac{\sqrt{2\log(1.25/\delta)}}{\epsilon}$, generate a $p\times p$ symmetric noisy matrix $E$ whose entries are $i.i.d.$ drawn from $\mathcal{N}(0,\beta^2)$.
        \STATE $\widetilde{S}\leftarrow S+E$.
        \STATE Compute $\widehat{\Theta}_{priv}$ by repeat \eqref{theta-update}-\eqref{u-update} based on $\widetilde{S}$ until convergence.
	\end{algorithmic}  
\end{algorithm}

\begin{theorem}
    Algorithms \ref{alg:1} and \ref{alg:2} are $(\epsilon,\delta)$-differentially private.
\end{theorem}

\begin{proof}
By the definition of the Gaussian mechanism, we know that step 3 keeps the matrix $(\epsilon,\delta)$-differentially private. Thus, due to the post-processing property of differential privacy, Algorithm \ref{alg:1} is $(\epsilon,\delta)$-differentially private, and Algorithm \ref{alg:2} is also.	
\end{proof}

\begin{remark}
	 It should be noted that the noise addition mechanism in Algorithm \ref{alg:2} is the same as \cite{wang2018differentially}, but we derived it from the perspective of the algorithm. The differential privacy of Algorithm \ref{alg:2} is achieved by the differential privacy of $Z$-minimization step, so the different privacy mechanism of this step will determine the different noise addition mechanism in Algorithm \ref{alg:2}. Moreover, this method of developing the differential privacy algorithm can be extended to the development of the differential privacy algorithm for the regularized empirical risk minimization problem.
	 
\end{remark}

\section{Theoretical analysis}\label{Theoretical analysis}

In this section, we provide theoretical guarantees on the performance of privacy-preserving precision matrix estimation algorithms in Section \ref{methods}. 

\subsection{Differentially private ridge estimation for the precision matrix}

We first provide performance theorem on Algorithm \ref{alg:1}. We will show that the differentially private ridge estimator for the precision matrix estimation is consistent under fixed-dimension asymptotic. 
\begin{theorem}\label{Theorem4.1}
	Let $\widehat{\Theta}_{priv}$ be the output of Algorithm \ref{alg:1}. Let $\Theta^*$ be the true precision matrix. Then, if tuning parameter $\lambda$ converges almost surely to zero as $n\to \infty$, with high probability, we have 
	\begin{equation}
		\lim_{n\to \infty} \mathbb{E}\left( \| \widehat{\Theta}_{priv} - \Theta^* \|_F^2 \right) = 0.
	\end{equation}
\end{theorem}

In this proof, we will need the following lemmas. The first had been proved in  \cite{rolfs2012iterative}.

\begin{lemma}[\cite{rolfs2012iterative}]\label{Lemma3}
	Let $\Theta_1$, $\Theta_2$ be positive definite matrices. For any $\zeta > 0$, we have 
	\begin{equation}
		\|\Theta_1 + \zeta \Theta_1^{-1} -(\Theta_2 + \zeta \Theta_2^{-1}) \|_F \leq \max\{ |1-\frac{\zeta}{a^2}, |1-\frac{\zeta}{b^2}|\}\|\Theta_1 - \Theta_2 \|_F,
	\end{equation}
	where $a=\max\{\varphi_{\max}(\Theta_1), \varphi_{\max}(\Theta_2)\}$, $b=\min\{\varphi_{\min}(\Theta_1), \varphi_{\min}(\Theta_2)\}$. Here $\varphi_{\max}$ and $\varphi_{\min}$ represent the maximum eigenvalue and the minimum eigenvalue, respectively.
\end{lemma}

The second comes from van Wieringen and Peeters \cite{vanWieringen2016ridge}, which shows the consistency of the ridge estimator \eqref{rope} under fixed-dimension asymptotic.

\begin{lemma}[\cite{vanWieringen2016ridge}]\label{consistency}
	Let $\widehat{\Theta}$ be the optimal solution of problem \eqref{gridge} and $\Theta^*$ be the true precision matrix. Then, if tuning parameter $\lambda$ converges almost surely to zero as $n\to \infty$, we have 
	\begin{equation}
		\lim_{n\to \infty} \mathbb{E}\left( \| \widehat{\Theta} - \Theta^* \|_F^2 \right) = 0.
	\end{equation}

\end{lemma}

We now prove the Theorem 4.1.

\begin{proof}
	By the triangular inequality, 
	\begin{equation}\label{main_ineq2}
		\|\widehat{\Theta}_{priv} - \Theta^*\|_F \leq \|\widehat{\Theta}_{priv} - \widehat{\Theta}\|_F + \| \widehat{\Theta} - \Theta^*\|_F,
	\end{equation}
	here $\widehat{\Theta}$ is the optimal solution of problem \eqref{gridge}, which is \eqref{rope}. 
	
	First, we bound the first term on the right-hand side of \eqref{main_ineq2}. Let $g=\lambda\|\Theta\| _F^2$ and $\mbox{prox}_g$ denote the proximity operator of $g$. Then for any $\zeta > 0$,
	\begin{eqnarray}
		\|\widehat{\Theta}_{priv} - \widehat{\Theta}\|_F  &=& \| \mbox{prox}_{\zeta g}(\widehat{\Theta}_{priv} - \zeta (\widetilde{S} - \widehat{\Theta}_{priv}^{-1})) - \mbox{prox}_{\zeta g}(\widehat{\Theta} - \zeta (S - \widehat{\Theta}^{-1})) \|_F \nonumber \\
		&\leq & \| \widehat{\Theta}_{priv} - \zeta (\widetilde{S} - \widehat{\Theta}_{priv}^{-1}) - (\widehat{\Theta} - \zeta (S - \widehat{\Theta}^{-1})) \|_F \nonumber \\
		&=& \| \widehat{\Theta}_{priv} + \zeta \widehat{\Theta}_{priv}^{-1} - (\widehat{\Theta} + \zeta \widehat{\Theta}^{-1}) \|_F + \zeta \|\widetilde{S} -S \|_F.
	\end{eqnarray}
	The inequality uses the non-expansive property of the proximity operator \cite{combettes2005signal}. By Lemma \ref{Lemma3} and take $0 < \zeta <\min\{\varphi_{\min}(\widehat{\Theta}_{priv}), \varphi_{\min}(\widehat{\Theta})\}$, we have 
	\begin{equation}
		\|\widehat{\Theta}_{priv} - \widehat{\Theta}\|_F \leq \max\{\|\widehat{\Theta}\|_2^2, \|\widehat{\Theta}_{priv}\|_2^2\} \|E\|_F.
	\end{equation}
	Note that the entries in $E$ are $i.i.d.$ drawn from $\mathcal{N}(0,\beta^2)$. Then, with high probability, we have
	\begin{equation}\label{firstterm}
		\|\widehat{\Theta}_{priv} - \widehat{\Theta}\|_F \leq \mathcal{O}(\frac{p\sqrt{\log (1/\delta)}\max\{\|\widehat{\Theta}\|_2^2, \|\widehat{\Theta}_{priv}\|_2^2\}}{\epsilon n}).
	\end{equation}
	Take expectations on both sides of \eqref{firstterm},
	\begin{equation}\label{Efirstterm}
		\mathbb{E}\left(\|\widehat{\Theta}_{priv} - \widehat{\Theta}\|_F \right)
		\leq \mathcal{O}(\frac{p\sqrt{\log (1/\delta)}\mathbb{E}\left( \max\{\|\widehat{\Theta}\|_2^2, \|\widehat{\Theta}_{priv}\|_2^2\}\right)}{\epsilon n} ).
	\end{equation}
	Note that $\widehat{\Theta}$ and $\widehat{\Theta}_{priv}$ are always achieved, and thus $\mathbb{E}\left( \max\{\|\widehat{\Theta}\|_2^2, \|\widehat{\Theta}_{priv}\|_2^2\}\right)$ is finite. 
	
	Take expectations on both sides of \eqref{main_ineq2}, then we arrive the conclusion by \eqref{Efirstterm} and Lemma \ref{consistency}.
\end{proof}

\subsection{Differentially private algorithm for the graphical lasso}

We now establish performance guarantees for Algorithm \ref{alg:2}, by providing a rate of decay of its error in Frobenius norm. Generally, let $\widehat{\Theta}_{priv}$ be the optimal solution of problem \eqref{glasso} based on $\widetilde{S}$. Let set $S_1=\{(i,j):\Theta^*_{ij}\neq 0, i\neq j \}$, with the cardinality $s$, denotes non-zero off-diagonal entries in the precision matrix and $S_1^C$ be its complement. We provide the result as the following theorem.

\begin{theorem}
    Let $\widehat{\Theta}_{priv}$ be the optimal solution of problem \eqref{glasso} based on $\widetilde{S}$ with tuning parameter $\lambda \asymp \frac{p^2\sqrt{\log (1/\delta)}}{\epsilon n}+\sqrt{\frac{\log p}{n}}$  . Let $\Theta^*$ be the true precision matrix. Then, if the smallest and largest eigenvalues of the covariance matrix satisfy  $0 < k \leq \varphi_{\min}(\Sigma^*) \leq \varphi_{\max}(\Sigma^*) \leq k^{-1} < \infty$, with high probability, we have 
    \begin{equation}
    	\|\widehat{\Theta}_{priv}-\Theta^*\|_F = \mathcal{O}(\frac{p^2\sqrt{s\log (1/\delta)}}{\epsilon n} + \sqrt{\frac{s\log p}{n}}).
    \end{equation} 
\end{theorem}

In this proof, we will need a lemma of Bickel and Levina \cite{bickel2008regularized}.

\begin{lemma}[\cite{bickel2008regularized}]\label{tail}
	Let $Z_i$ be $i.i.d.$ $\mathcal{N}(0,\Sigma_p)$ and $\varphi_{\max}(\Sigma_p) \leq k^{-1} < \infty$. Then if $\Sigma_p=[\sigma_{ab}],$
	\begin{equation}
		P \left[ \left| \sum_{i=1}^n (Z_{ij}Z_{ik}-\sigma_{jk}) \right| \geq nv \right]\leq c_1\exp(-c_2nv^2) \quad for \;|v|\leq \delta, 
	\end{equation}
	where $c_1$, $c_2$ and $\delta$ depend on $k$ only.
\end{lemma}

We now prove the Theorem 4.2.

\begin{proof}
Let
	\begin{eqnarray}{\label{Qfunction}}
		Q(\Theta)& = &-\log |\Theta| +\mbox{tr}(\widetilde{S} \Theta) +\lambda \| \Theta\| _1+\log |\Theta^*| -\mbox{tr}(\widetilde{S} \Theta^*) -\lambda \| \Theta^*\| _1 \nonumber \\
		& = &-(\log |\Theta| - \log |\Theta^*|) + \mbox{tr}[(\Theta-\Theta^*)(\widetilde{S}-\Sigma^*)] + \mbox{tr}[(\Theta-\Theta^*)\Sigma^*] \nonumber \\
		&   &+ \lambda(\| \Theta\| _1 - \| \Theta^*\| _1), 
	\end{eqnarray}
	The estimate $\widehat{\Theta}_{priv}$ minimizes $Q(\Theta$), or equivalently $\widehat{\Delta}=\widehat{\Theta}_{priv}-\Theta^*$ minimizes $G(\Delta) \equiv Q(\Theta^*+\Delta)$. 
	
	Consider the set
	\begin{equation}
		J_n(M)=\{\Delta : \Delta=\Delta^T, \|\Delta\|_F = Mr_n \}, \qquad \mbox{where} \; r_n = \frac{p^2\sqrt{s\log (1/\delta)}}{\epsilon n} + \sqrt{\frac{s\log p}{n}}  \rightarrow 0.
	\end{equation}
	Note that $G(\Delta) = Q(\Theta^*+\Delta)$ is a convex function, and $G(\widehat{\Delta}) \leq G(0) = 0 $. Then, if we can show that 
	\begin{equation}
		\inf\{G(\Delta) : \Delta \in J_n(M) \} > 0,
	\end{equation}
	the minimizer $\widehat{\Delta}$ must be inside the sphere defined by $J_n(M)$, and hence 
	\begin{equation}
		\|\widehat{\Delta} \|_F \leq Mr_n.
	\end{equation}
	
	By the Taylor expansion of $f(t)=\log |\Theta + t\Delta|$, we rewrite \eqref{Qfunction} as, 
	\begin{eqnarray}
		G(\Delta)&=& \widetilde{\Delta}^T \left[ \int_0^1(1-v)(\Theta^*+v\Delta)^{-1} \otimes (\Theta^*+v\Delta)^{-1}dv \right] \widetilde{\Delta} + \mbox{tr}(\Delta (\widetilde{S}-\Sigma^*)) \nonumber \\
		&  &+\lambda(\| \Theta^* + \Delta \| _1 - \| \Theta^*\| _1) \nonumber \\
		&= &\mbox{I}+\mbox{II}+\mbox{III},
	\end{eqnarray}
	where $\otimes$ is the Kronecker product and $\widetilde{\Delta}$ is $\Delta$ vectorized to match the dimensions of the Kronecker product. 
	
	Rothman et al. \cite{rothman2008sparse} had show that term $\mbox{I} \geq \frac{1}{4}k^2$ with high probability. 
	
	To bound term $\mbox{II}$, we write 
	\begin{equation}
		| \mbox{tr}(\Delta (\widetilde{S}-\Sigma^*)) | \leq \left| \sum_{i,j}(S_{ij}-\Sigma^*_{ij})\Delta_{ij} \right|.
	\end{equation}
	Note that the union sum inequality and Lemma \ref{tail} imply that, with high probability, 
	\begin{equation}
		\max_{i,j} |\widetilde{S}_{ij}-\Sigma^*_{ij}| \leq  \max_{i,j} |S_{ij}-\Sigma^*_{ij}|+\max_{i,j}|E_{ij}| \leq c_1\sqrt{\frac{\log p}{n}} + c_2\frac{p^2\sqrt{\log (1/\delta)}}{\epsilon n},
	\end{equation}
	and hence term $\mbox{II}$ is bounded by
	\begin{equation}
		\mbox{II} \geq (-c_1\sqrt{\frac{\log p}{n}} - c_2\frac{p^2\sqrt{\log (1/\delta)}}{\epsilon n})\|\Delta\|_1,
	\end{equation}
	here $c_1, c_2$ are positive constants.
	
	Note that $\| \Theta^* + \Delta \| _1 = \| \Theta^*_{S_1} + \Delta_{S_1} \| _1 + \|\Theta^*_{S_1^C}\|_1 + \|\Delta_{S_1^C}\|_1$ and $\|\Theta^*\|_1 = \|\Theta^*_{S_1}\|_1+\|\Theta^*_{S_1^C}\|_1$. Then by triangular inequality, the third part 
	\begin{equation}
		\mbox{III} = \lambda(\| \Theta^*_{S_1} + \Delta_{S_1} \| _1 + \|\Delta_{S_1^C}\|_1 -  \|\Theta^*_{S_1}\|_1 ) \geq \lambda(\|\Delta_{S_1^C}\|_1-\|\Delta_{S_1}\|_1).
	\end{equation}
	
	Therefore, we have
	\begin{equation}
		G(\Delta) \geq \frac{1}{4}k^2\|\Delta\|_F^2 - (c_1\sqrt{\frac{\log p}{n}} + c_2\frac{p^2\sqrt{\log (1/\delta)}}{\epsilon n})\|\Delta\|_1 + \lambda(\|\Delta_{S_1^C}\|_1-\|\Delta_{S_1}\|_1).
	\end{equation}
	Now, take 
	\begin{equation}
		\lambda = \frac{1}{\alpha}(c_1\sqrt{\frac{\log p}{n}} + c_2\frac{p^2\sqrt{\log (1/\delta)}}{\epsilon n}), 
	\end{equation}
	where $\alpha$ belongs to $(0, 1)$. Then
	
	\begin{eqnarray}\label{lowerbound}
		G(\Delta) &\geq& \frac{1}{4}k^2\|\Delta\|_F^2 - (c_1\sqrt{\frac{\log p}{n}} + c_2\frac{p^2\sqrt{\log (1/\delta)}}{\epsilon n})(1-\frac{1}{\alpha})\|\Delta_{S_1^C}\|_1 \nonumber \\
		&-& (c_1\sqrt{\frac{\log p}{n}} + c_2\frac{p^2\sqrt{\log (1/\delta)}}{\epsilon n})(1+\frac{1}{\alpha})\|\Delta_{S_1}\|_1.
	\end{eqnarray}
	Since that the second term on the right-hand side of \eqref{lowerbound} is always negative and $\|\Delta_{S_1}\|_1 \leq \sqrt{s} \|\Delta_{S_1}\|_F \leq \sqrt{s} \|\Delta\|_F$, we then have
	\begin{eqnarray}
		G(\Delta)& \geq &\|\Delta\|_F^2 \left[\frac{1}{4}k^2 - (c_1\sqrt{\frac{s\log p}{n}} + c_2\frac{p^2\sqrt{s\log (1/\delta)}}{\epsilon n})(1+\frac{1}{\alpha})\|\Delta\|_F^{-1} \right]	\nonumber \\
		& = & 	\|\Delta\|_F^2 \left[\frac{1}{4}k^2 - \frac{c(1+\alpha)}{\alpha M} \right],
	\end{eqnarray}
	here $c=c_1 / c_2$.
	
	Consequently, $G(\Delta)>0$ for $M$ sufficiently large. This establishes the theorem.
\end{proof}

\section{Experiments} \label{experiments}

\subsection{Simulation}

In this subsection, we conduct a simulation study to evaluate the performance of our proposed differentially private algorithms. We consider four models as follows:

\begin{itemize}
    \item $Model\;1$. $\Theta=\frac{1}{n}WW^T$, where $W$ is a $p\times n$ matrix with $n=10000$ and each $w_{i,j}$ is drawn from $\mathcal{N}(0,1)$. This precision matrix is unstructured and non-sparse.
    \item $Model\;2$. Full model with $\theta_{i,j}=1$ if $i=j$ and $\theta_{i,j}=0.5$ otherwise. This precision matrix is structured and non-sparse.
    \item $Model\;3$. An AR(2) model with $\theta_{i,i}=0, \theta_{i,i-1}=\theta_{i-1,i}=0.5$ and $\theta_{i,i-2}=\theta_{i-2,i}=0.25$. This precision matrix is structured and sparse.
    \item $Model\;4$. The last model comes from  Cai et al. \cite{cai2011Aconstrained}. 
    Let $\Theta_0=A+\alpha I$, where each off-diagonal entry in $A$ is generated independently and equals 0.5 with probability 0.1 or 0 with probability 0.9. $\alpha$ is chosen such that the condition number (the ratio of the largest and the smallest eigenvalues of a matrix) of the matrix is equal to $p$. Finally, the matrix is standardized to have unit diagonal. This precision matrix is unstructured and sparse.
\end{itemize}

Firstly, we measure the difference between $\widehat{\Theta}_{priv}$ and $\widehat{\Theta}$, where $\widehat{\Theta}_{priv}$ denotes the output of our differentially private Algorithm \ref{alg:1} and Algorithm \ref{alg:2}, the corresponding $\widehat{\Theta}$ denotes the solution of the problem \eqref{gridge} solved by equation \eqref{rope} and the solution of the problem \eqref{glasso} solved by the ADMM algorithm. For each model, we generate an independent sample of size $n$ from a multivariate Gaussian distribution with mean zero and covariance matrix $\Sigma=\Theta^{-1}$, and preprocess the data matrix by normalizing the data matrix with the maximum $l_2$ norm to enforce the condition $\|x_i\|_2 \leq 1$. A five-fold cross-validation described by Bien and Tibshirani \cite{bien2011sparse} is used to choose the parameter such that the optimal value of $\lambda$ minimizes the negative log-likelihood \eqref{log-likelihood}, and we compute the differentially private estimations on the same tuning parameter $\lambda$ as the non-private algorithm. We fixed $p=100$ and consider different values of

\begin{table}
    \centering
    \scriptsize

    \vspace{0.2cm}
    \begin{tabular}{ccccccccccccccc}
	    \toprule

	    \multirow{2}{*}{$n$}& &\multirow{2}{*}{$\epsilon$}& &\multicolumn{3}{c}{Model1}& &\multicolumn{3}{c}{Model2}& &\multicolumn{3}{c}{Model3}\cr
	    \cmidrule(lr){5-7} \cmidrule(lr){9-11} \cmidrule(lr){13-15}  
	     & & & &$l_1\,norm$&$F\,norm$&$l_2\,norm$& &$l_1\,norm$&$F\,norm$&$l_2\,norm$& &$l_1\,norm$&$F\,norm$&$l_2\,norm$\cr
	    \hline
	    \multirow{12}{*}{$100$}\cr
	    & &\multirow{2}{*}{$0.1$}& &	2033.43	&	992.43	&	1952.80	&	&	2023.17	&	982.35	&	1937.40	&	&	2799.51	&	1397.73	&	2732.86	\cr
	    & &                      & &	(214.65)	&	(119.61)	&	(217.19)	&	&	(203.78)	&	(105.55)	&	(193.82)	&	&	(187.39)	&	(77.52)	&	(157.59)	\cr
	    & &\multirow{2}{*}{$0.3$}& &	676.69	&	329.80	&	90.76	&	&	673.34	&	326.44	&	89.84	&	&	932.06	&	464.65	&	128.87	\cr
	    & &                      & &	(71.54)	&	(39.83)	&	(11.44)	&	&	(67.87)	&	(35.14)	&	(10.15)	&	&	(62.48)	&	(25.81)	&	(7.55)	\cr
	    & &\multirow{2}{*}{$0.5$}& &	405.34	&	197.28	&	54.35	&	&	403.36	&	195.27	&	53.80	&	&	558.56	&	278.04	&	77.20	\cr
	    & &                      & &	(42.92)	&	(23.87)	&	(6.86)	&	&	(40.69)	&	(21.06)	&	(6.08)	&	&	(37.50)	&	(15.47)	&	(4.53)	\cr
	    & &\multirow{2}{*}{$0.8$}& &	252.70	&	122.74	&	33.88	&	&	251.50	&	121.48	&	33.53	&	&	348.46	&	173.07	&	48.14	\cr
	    & &                      & &	(26.82)	&	(14.89)	&	(4.28)	&	&	(25.40)	&	(13.14)	&	(3.80)	&	&	(23.45)	&	(9.65)	&	(2.83)	\cr
	    & &\multirow{2}{*}{$1.2$}& &	118.22	&	57.08	&	15.84	&	&	117.71	&	56.49	&	15.68	&	&	163.35	&	80.61	&	22.53	\cr
	    & &                      & &	(12.64)	&	(6.98)	&	(2.01)	&	&	(11.94)	&	(6.16)	&	(1.79)	&	&	(11.10)	&	(4.53)	&	(1.33)	\cr
	    & &\multirow{2}{*}{$2$}  & &	70.25	&	33.67	&	9.40	    &	&	69.98	&	33.33	&	9.31	&	&	97.30	&	47.65	&	13.40	\cr
	    & &                      & &	(7.58)	&	(4.16)	&	(1.20)	&	&	(7.13)	&	(3.67)	&	(1.07)	&	&	(6.70)	&	(2.70)	&	(0.80)	\cr
\\
	    \hline
	    \multirow{12}{*}{$200$}\cr
	    & &\multirow{2}{*}{$0.1$}& &	1640.76	&	773.12	&	1397.26	&	&	1639.64	&	782.21	&	1326.40	&	&	2634.83	&	1315.41	&	2277.36	\cr
	    & &                      & &	(102.37)	&	(39.26)	&	(76.36)	&	&	(95.85)	&	(38.54)	&	(69.81)	&	&	(171.50)	&	(82.98)	&	(156.16)	\cr
	    & &\multirow{2}{*}{$0.3$}& &	545.36	&	256.42	&	70.62	&	&	544.96	&	259.44	&	71.45	&	&	876.13	&	436.65	&	121.23	\cr
	    & &                      & &	(34.07)	&	(13.07)	&	(4.00)	&	&	(31.92)	&	(12.84)	&	(3.85)	&	&	(57.08)	&	(27.64)	&	(8.52)	\cr
	    & &\multirow{2}{*}{$0.5$}& &	326.27	&	153.08	&	42.24	&	&	326.02	&	154.89	&	42.74	&	&	524.37	&	260.90	&	72.56	\cr
	    & &                      & &	(20.41)	&	(7.84)	&	(2.40)	&	&	(19.14)	&	(7.69)	&	(2.31)	&	&	(34.21)	&	(16.57)	&	(5.11)	\cr
	    & &\multirow{2}{*}{$0.8$}& &	203.03	&	94.97	&	26.28	&	&	202.85	&	96.09	&	26.59	&	&	326.49	&	162.05	&	45.18	\cr
	    & &                      & &	(12.73)	&	(4.89)	&	(1.50)	&	&	(11.96)	&	(4.80)	&	(1.44)	&	&	(21.36)	&	(10.35)	&	(3.19)	\cr
	    & &\multirow{2}{*}{$1.2$}& &	94.44	&	43.78	&	12.22	&	&	94.34	&	44.31	&	12.37	&	&	152.20	&	74.99	&	21.06	\cr
	    & &                      & &	(5.96)	&	(2.29)	&	(0.70)	&	&	(5.63)	&	(2.25)	&	(0.68)	&	&	(10.12)	&	(4.86)	&	(1.50)	\cr
	    & &\multirow{2}{*}{$2$}  & &	55.70	&	25.55	&	7.20	    &	&	55.63	&	25.86	&	7.29	&	&	90.01	&	43.97	&	12.46	\cr
	    & &                      & &	(3.56)	&	(1.37)	&	(0.42)	&	&	(3.37)	&	(1.34)	&	(0.41)	&	&	(6.08)	&	(2.90)	&	(0.90)	\cr
\\
	    \hline
	    \multirow{12}{*}{$400$}\cr
	    & &\multirow{2}{*}{$0.1$}& &	1328.60	&	578.15	&	1040.17	&	&	1347.95	&	592.49	&	859.38	&	&	2589.53	&	1208.54	&	1934.99	\cr
	    & &                      & &	(73.95)	&	(31.23)	&	(59.22)	&	&	(75.70)	&	(31.00)	&	(45.65)	&	&	(105.15)	&	(39.34)	&	(62.87)	\cr
	    & &\multirow{2}{*}{$0.3$}& &	440.65	&	191.07	&	52.75	&	&	447.01	&	195.82	&	54.04	&	&	859.69	&	400.10	&	111.19	\cr
	    & &                      & &	(24.63)	&	(10.39)	&	(3.14)	&	&	(25.17)	&	(10.31)	&	(3.03)	&	&	(35.10)	&	(13.00)	&	(4.11)	\cr
	    & &\multirow{2}{*}{$0.5$}& &	263.04	&	113.67	&	31.49	&	&	266.81	&	116.50	&	32.26	&	&	513.76	&	238.43	&	66.43	\cr
	    & &                      & &	(14.76)	&	(6.22)	&	(1.88)	&	&	(15.07)	&	(6.18)	&	(1.82)	&	&	(21.15)	&	(7.73)	&	(2.46)	\cr
	    & &\multirow{2}{*}{$0.8$}& &	163.13	&	70.14	&	19.53	&	&	165.43	&	71.89	&	20.00	&	&	319.19	&	147.51	&	41.25	\cr
	    & &                      & &	(9.21)	&	(3.88)	&	(1.17)	&	&	(9.39)	&	(3.85)	&	(1.13)	&	&	(13.34)	&	(4.77)	&	(1.53)	\cr
	    & &\multirow{2}{*}{$1.2$}& &	75.09	&	31.83	&	9.00 	&	&	76.12	&	32.63	&	9.21	&	&	147.72	&	67.46	&	19.08	\cr
	    & &                      & &	(4.32)	&	(1.81)	&	(0.55)	&	&	(4.37)	&	(1.80)	&	(0.53)	&	&	(6.45)	&	(2.17)	&	(0.72)	\cr
	    & &\multirow{2}{*}{$2$}  & &	43.70	&	18.21	&	5.24 	&	&	44.27	&	18.67	&	5.37	&	&	86.54	&	38.97	&	11.18	\cr
	    & &                      & &	(2.60)	&	(1.07)	&	(0.33)	&	&	(2.57)	&	(1.06)	&	(0.32)	&	&	(3.94)	&	(1.24)	&	(0.42)	\cr
\\
	    \bottomrule
	\end{tabular}
	\caption{Simulation results for differentially private Algorithm \ref{alg:1} in Model 1-3. Average (standard error) matrix losses over 50 replications.\label{tab:1}}
\end{table}

\begin{table}
    \centering
    \scriptsize

    \vspace{0.2cm}
    \begin{tabular}{ccccccccccccccc}
        \toprule
    
        \multirow{2}{*}{$n$}& &\multirow{2}{*}{$\epsilon$}& &\multicolumn{3}{c}{Model2}& &\multicolumn{3}{c}{Model3}& &\multicolumn{3}{c}{Model4}\cr
        \cmidrule(lr){5-7} \cmidrule(lr){9-11} \cmidrule(lr){13-15}  
         & & & &$l_1\,norm$&$F\,norm$&$l_2\,norm$& &$l_1\,norm$&$F\,norm$&$l_2\,norm$& &$l_1\,norm$&$F\,norm$&$l_2\,norm$\cr
        \hline

        \multirow{12}{*}{$100$}\cr
         & &\multirow{2}{*}{$0.1$}& &	82.90	&	9.45	&	26.23	&	&	80.19	&	9.45	&	26.09	&	&	67.15	&	9.42	&	26.02	\cr
         & &                      & &	(2.78)	&	(0.17)	&	(0.83)	&	&	(2.72)	&	(0.16)	&	(0.83)	&	&	(3.38)	&	(0.16)	&	(0.83)	\cr
         & &\multirow{2}{*}{$0.3$}& &	27.27	&	2.98	&	0.83	&	&	26.45	&	2.99	&	0.83	&	&	22.18	&	2.98	&	0.83	\cr
         & &                      & &	(0.94)	&	(0.05)	&	(0.03)	&	&	(0.92)	&	(0.05)	&	(0.03)	&	&	(1.14)	&	(0.05)	&	(0.03)	\cr
         & &\multirow{2}{*}{$0.5$}& &	16.31	&	1.75	&	0.48	&	&	15.87	&	1.75	&	0.48	&	&	13.32	&	1.75	&	0.48	\cr
         & &                      & &	(0.57)	&	(0.03)	&	(0.02)	&	&	(0.56)	&	(0.03)	&	(0.02)	&	&	(0.69)	&	(0.03)	&	(0.02)	\cr
         & &\multirow{2}{*}{$0.8$}& &	10.22	&	1.08	&	0.28	&	&	9.98	&	1.09	&	0.29	&	&	8.39	&	1.09	&	0.29	\cr
         & &                      & &	(0.36)	&	(0.02)	&	(0.01)	&	&	(0.35)	&	(0.02)	&	(0.01)	&	&	(0.43)	&	(0.02)	&	(0.01)	\cr
         & &\multirow{2}{*}{$1.2$}& &	4.85	    &	0.52	&	0.12	&	&	4.80	&	0.52	&	0.12	&	&	4.05	&	0.53	&	0.12	\cr
         & &                      & &	(0.17)	&	(0.01)	&	(0.00)	&	&	(0.17)	&	(0.01)	&	(0.00)	&	&	(0.20)	&	(0.01)	&	(0.00)	\cr
         & &\multirow{2}{*}{$2$}  & &	2.88 	&	0.31	&	0.07	&	&	2.89	&	0.32	&	0.07	&	&	2.46	&	0.32	&	0.07	\cr
         & &                      & &	(0.10)	&	(0.00)	&	(0.00)	&	&	(0.10)	&	(0.00)	&	(0.00)	&	&	(0.12)	&	(0.00)	&	(0.00)	\cr
         \\
         \hline

    	\multirow{12}{*}{$200$}\cr
    	 & &\multirow{2}{*}{$0.1$}& &	44.10	&	4.91	&	13.68	&	&	42.76	&	4.91	&	13.61	&	&	36.09	&	4.90	&	13.59	\cr
	     & &                      & &	(1.49)	&	(0.09)	&	(0.44)	&	&	(1.40)	&	(0.09)	&	(0.44)	&	&	(1.67)	&	(0.09)	&	(0.44)	\cr
 	     & &\multirow{2}{*}{$0.3$}& &	14.61	&	1.55	&	0.42	&	&	14.23	&	1.56	&	0.42	&	&	12.03	&	1.56	&	0.42	\cr
    	 & &                      & &	(0.51)	&	(0.03)	&	(0.01)	&	&	(0.48)	&	(0.03)	&	(0.01)	&	&	(0.57)	&	(0.03)	&	(0.01)	\cr
     	 & &\multirow{2}{*}{$0.5$}& &	8.83 	&	0.93	&	0.24	&	&	8.63	&	0.94	&	0.24	&	&	7.31	&	0.94	&	0.24	\cr
 	     & &                      & &	(0.31)	&	(0.01)	&	(0.01)	&	&	(0.29)	&	(0.01)	&	(0.01)	&	&	(0.35)	&	(0.01)	&	(0.01)	\cr
	     & &\multirow{2}{*}{$0.8$}& &	5.57 	&	0.59	&	0.14	&	&	5.48	&	0.59	&	0.14	&	&	4.65	&	0.59	&	0.14	\cr
	     & &                      & &	(0.20)	&	(0.01)	&	(0.00)	&	&	(0.18)	&	(0.01)	&	(0.00)	&	&	(0.22)	&	(0.01)	&	(0.00)	\cr
	     & &\multirow{2}{*}{$1.2$}& &	2.61	    &	0.28	&	0.06	&	&	2.62	&	0.29	&	0.06	&	&	2.23	&	0.29	&	0.06	\cr
	     & &                      & &	(0.09)	&	(0.00)	&	(0.00)	&	&	(0.08)	&	(0.00)	&	(0.00)	&	&	(0.10)	&	(0.00)	&	(0.00)	\cr
	     & &\multirow{2}{*}{$2$}  & &	1.52	    &	0.16	&	0.03	&	&	1.56	&	0.17	&	0.04	&	&	1.34	&	0.17	&	0.04	\cr
	     & &                      & &	(0.05)	&	(0.00)	&	(0.00)	&	&	(0.05)	&	(0.00)	&	(0.00)	&	&	(0.06)	&	(0.00)	&	(0.00)	\cr
	     \\
  	    \hline
    
	    \multirow{12}{*}{$400$}\cr
	     & &\multirow{2}{*}{$0.1$}& &	23.31	&	2.52	&	6.98	&	&	22.74	&	2.52	&	6.95	&	&	19.29	&	2.52	&	6.94	\cr
	     & &                      & &	(0.80)	&	(0.04)	&	(0.23)	&	&	(0.79)	&	(0.04)	&	(0.23)	&	&	(0.79)	&	(0.04)	&	(0.23)	\cr
	     & &\multirow{2}{*}{$0.3$}& &  	7.87	    &	0.83	&	0.21	&	&	7.72	&	0.83	&	0.21	&	&	6.57	&	0.83	&	0.21	\cr
	     & &                      & &	(0.28)	&	(0.01)	&	(0.01)	&	&	(0.27)	&	(0.01)	&	(0.01)	&	&	(0.27)	&	(0.01)	&	(0.01)	\cr
	     & &\multirow{2}{*}{$0.5$}& &	4.79 	&	0.50	&	0.12	&	&	4.72	&	0.51	&	0.12	&	&	4.02	&	0.51	&	0.12	\cr
	     & &                      & &	(0.16)	&	(0.01)	&	(0.00)	&	&	(0.16)	&	(0.01)	&	(0.00)	&	&	(0.16)	&	(0.01)	&	(0.00)	\cr
	     & &\multirow{2}{*}{$0.8$}& &	3.00	    &	0.32	&	0.07	&	&	2.98	&	0.32	&	0.07	&	&	2.55	&	0.32	&	0.07	\cr
	     & &                      & &	(0.10)	&	(0.00)	&	(0.00)	&	&	(0.10)	&	(0.00)	&	(0.00)	&	&	(0.10)	&	(0.00)	&	(0.00)	\cr
	     & &\multirow{2}{*}{$1.2$}& &	1.37	    &	0.15	&	0.03	&	&	1.40	&	0.15	&	0.03	&	&	1.21	&	0.15	&	0.03	\cr
	     & &                      & &	(0.05)	&	(0.00)	&	(0.00)	&	&	(0.05)	&	(0.00)	&	(0.00)	&	&	(0.05)	&	(0.00)	&	(0.00)	\cr
	     & &\multirow{2}{*}{$2$}  & &	0.79	    &	0.09	&	0.02	&	&	0.83	&	0.09	&	0.02	&	&	0.72	&	0.09	&	0.02	\cr
	     & &                      & &	(0.03)	&	(0.00)	&	(0.00)	&	&	(0.03)	&	(0.00)	&	(0.00)	&	&	(0.03)	&	(0.00)	&	(0.00)	\cr
	     \\
	    \bottomrule
	\end{tabular}
    \caption{Simulation results for differentially private Algorithm \ref{alg:2} in Model 2-4. Average (standard error) matrix losses over 50 replications.\label{tab:2}}
\end{table}
$n = 100, 200, 400$ and change $\epsilon = 0.3, 0.5, 0.8, 1.2, 2$. For each experiments, we set $\delta = 1/n$ and replicate 50 times. In addition, the preselected penalty coefficient is set to $\rho=100$, We measure the estimation quality by the following matrix norms: the matrix $l_1$ norm $\|\widehat{\Theta}_{priv} - \widehat{\Theta}\| _1 / \|\widehat{\Theta}\| _1$, the Frobenius norm $\|\widehat{\Theta}_{priv} - \widehat{\Theta}\| _F / \|\widehat{\Theta}\| _F$, and the spectral norm $\|\widehat{\Theta}_{priv} - \widehat{\Theta}\| _2 / \|\widehat{\Theta}\| _2$. Tabel \ref{tab:1} and Tabel \ref{tab:2} report the averages and standard errors of these losses.

As shown in Tabel \ref{tab:1} and Tabel \ref{tab:2}, the high sample size and high privacy budget setting will have smaller matrix norm losses. First, we study how the loss of our differentially private algorithms is affected by $\epsilon$. As $\epsilon$ increase, these losses are reduced in both structured or unstructured, sparse or non-sparse models. This is in line with our additive-noise algorithm, where noise level is inversely proportional to privacy parameter, that is, larger $\epsilon$ means less noise and more utility. Next, we investigate the performance of our differentially private algorithms when the sample size varies. As the sample size $n$ increases while $\epsilon$ being fixed, all three losses show increasing estimation performance.

\begin{figure*}{}
    \centering
    \subfigure[Model 3]{\includegraphics[height=6cm,width=6cm,angle=0]{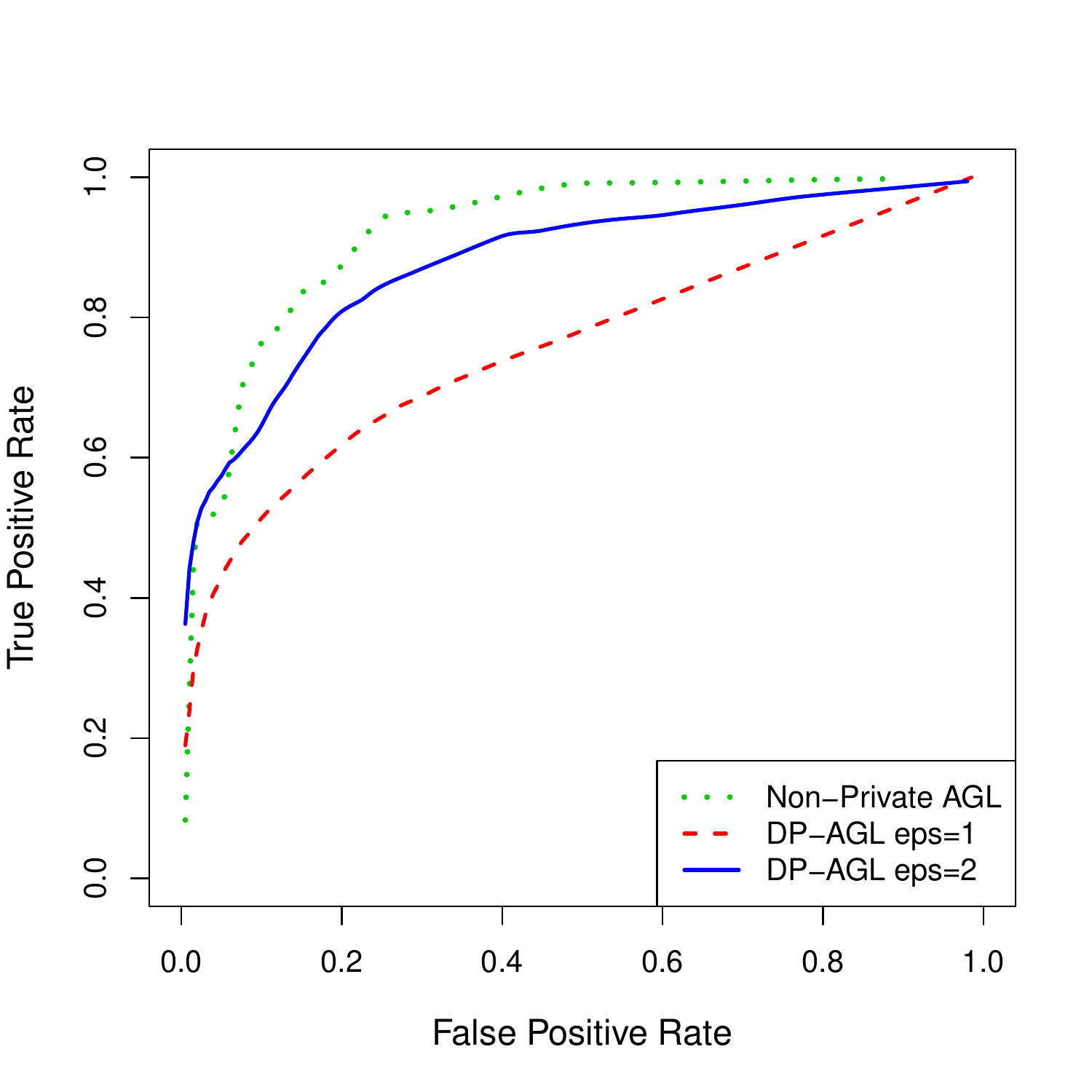}}
    \hspace{0.1\textwidth}
    \subfigure[Model 4]{\includegraphics[height=6cm,width=6cm,angle=0]{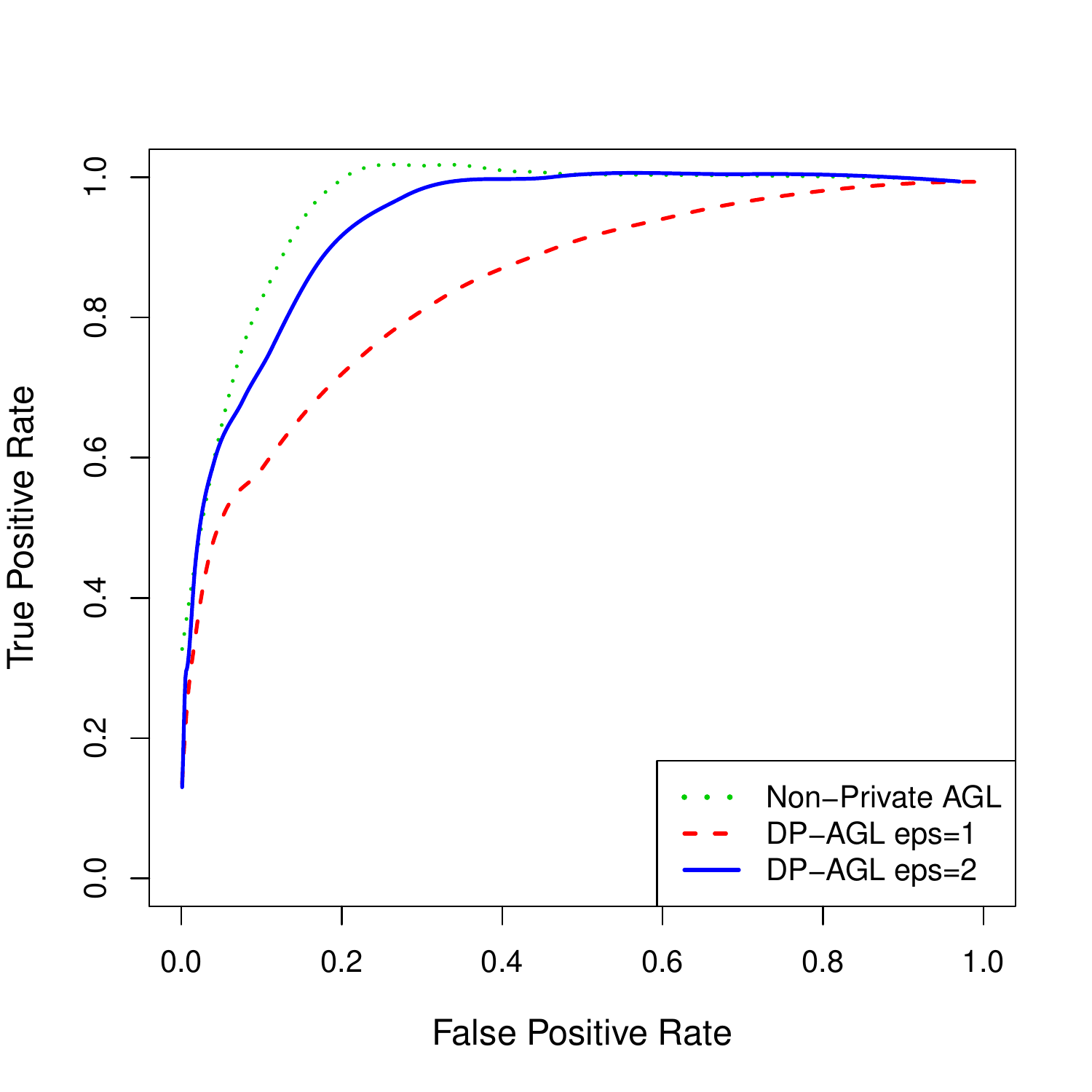}}
    \caption{ROC curves for varying levels of privacy parameter $\epsilon$, compared with the non-private algorithm. }
    \label{fig:1}
\end{figure*}

Secondly, focus on the structural learning of the Gaussian graphical model, we test the ability of our differentially private Algorithm \ref{alg:2} to recover the support of the precision matrix by ROC curves. ROC curves reflect the overall selection performance of each method as the tuning parameter varies, in which the true positive rate (TPR) is plotted against the false positive rate (FPR). For Model 3 and Model 4, we generate a sample of size $n=2000$ and preprocess the data matrix. We fixed $p=20, \delta= 0.001, \rho=100$ and let $\epsilon$ vary in $\{1,2\}$. For each setting, we replicate 50 times. Smoothed average ROC curves are shown in Figure \ref{fig:1}. From Figure \ref{fig:1}, we see that the model selection ability of our differentially private Algorithm \ref{alg:2} is almost comparable to the non-private algorithm with proper parameters setting.

\subsection{Application to real data}

\subsubsection{Classification of Ionosphere Data}

The first example deals with the Ionosphere data from UCI repository \cite{dua2017uci}. This radar data was collected by a system in Goose Bay, Labrador. This system consists of a phased array of 16 high-frequency antennas. The targets were free electrons in the ionosphere. ``Good'' radar returns are those showing evidence of some type of structure in the ionosphere. ``Bad'' returns are those that do not, their signals pass through the ionosphere. The data contains 351 observations with 34 variables. The response is labelled as 1 for ``Good'' and $-1$ for ``Bad''. As a classification problem, we apply linear discriminant analysis (LDA) with an differentially private precision matrix estimate. The purpose here is to illustrate how the ridge estimation of the precision matrix under different privacy parameters $\epsilon$ can affect the classification performance of LDA. We randomly divide the data into a training set, a validation set and a test set. The training set is used to estimate the precision matrix. The validation set is used to choose the tuning parameter for the method. The misclassification error is computed based on the test set. The sizes of the training set and the validation set are chosen to be n = 40. We set $\delta=0.001$ and repeat this procedure 50 times. The misclassification errors are present Figure \ref{fig:2}. Figure \ref{fig:2} shows that there is a large gap between the result of the differentially private algorithm and the result of the non-private algorithm when $\epsilon$ is small, but the gap becomes smaller as $\epsilon$ increases.

\begin{figure} 
    \centering
    \includegraphics[height=4.5cm,width=12cm,angle=0]{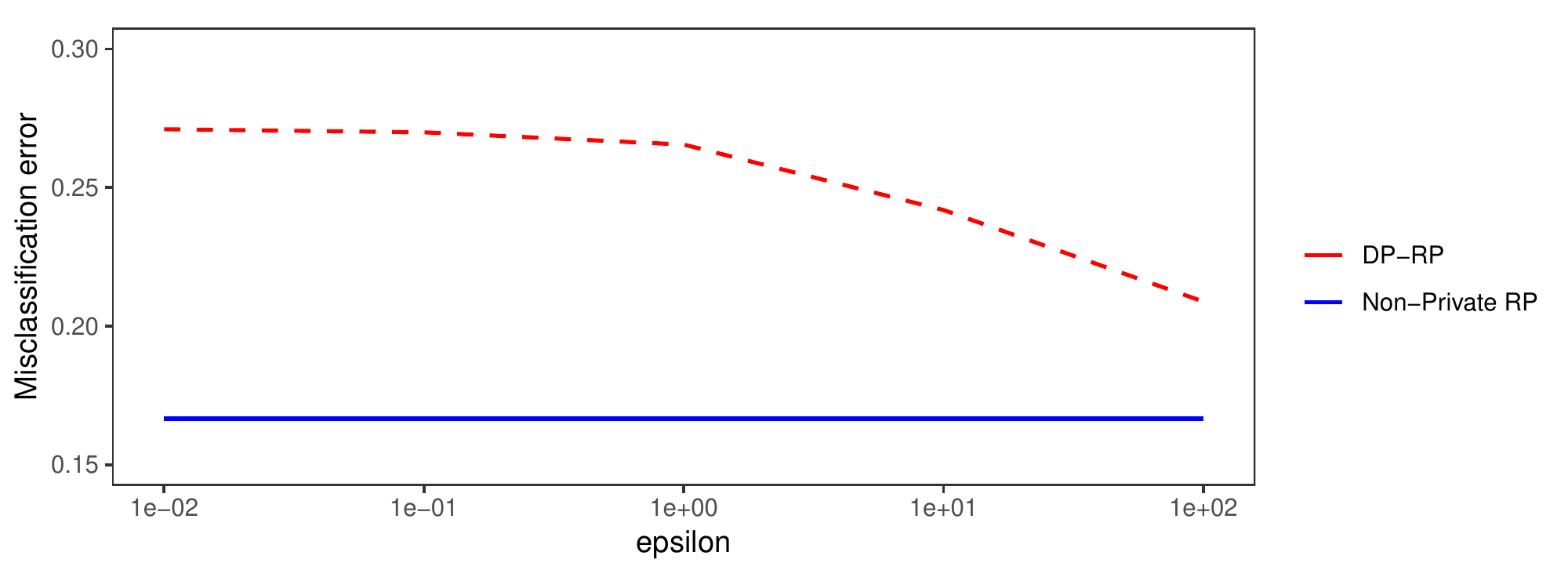}
    \caption{The misclassification errors of LDA}
    \label{fig:2}
\end{figure}

\subsubsection{Cell signalling data}

In this subsection, we apply our differentially private Algorithm \ref{alg:2} to the cell signalling data \cite{sachs2005causal} to evaluate its performance. The dataset containing 7466 cells, with flow cytometry measurements of 11 phosphorylated proteins and phospholipids, i.e., $n=7466$ and $p = 11$. They also provided a causal protein-signaling network, shown in Figure \ref{fig:3}a. Friedman et al. \cite{friedman2008sparse} used to apply graphical lasso to this data.

We compare the performance of our differentially private Algorithm \ref{alg:2} and non-private algorithm in reconstructing network edges by the ROC curve. To remove the randomness of the additive noise, we replicate 50 times. We fixed preselected penalty coefficient $\rho=100$ and privacy parameter $\delta=0.001$. The resulting ROC curve is shown in Figure \ref{fig:3}b, where the network in Figure \ref{fig:3}a was used as a benchmark (here we regard it as an undirected graph). The ROC curve shows that our differential privacy Algorithm \ref{alg:2} has high utility. We randomly show reconstructed networks with different tuning parameter $\lambda$ in Figure \ref{fig:4}.

\begin{figure*}[h]
    \centering
    \subfigure[]{\includegraphics[height=6cm,width=6cm,angle=0,trim=35 35 35 35,clip]{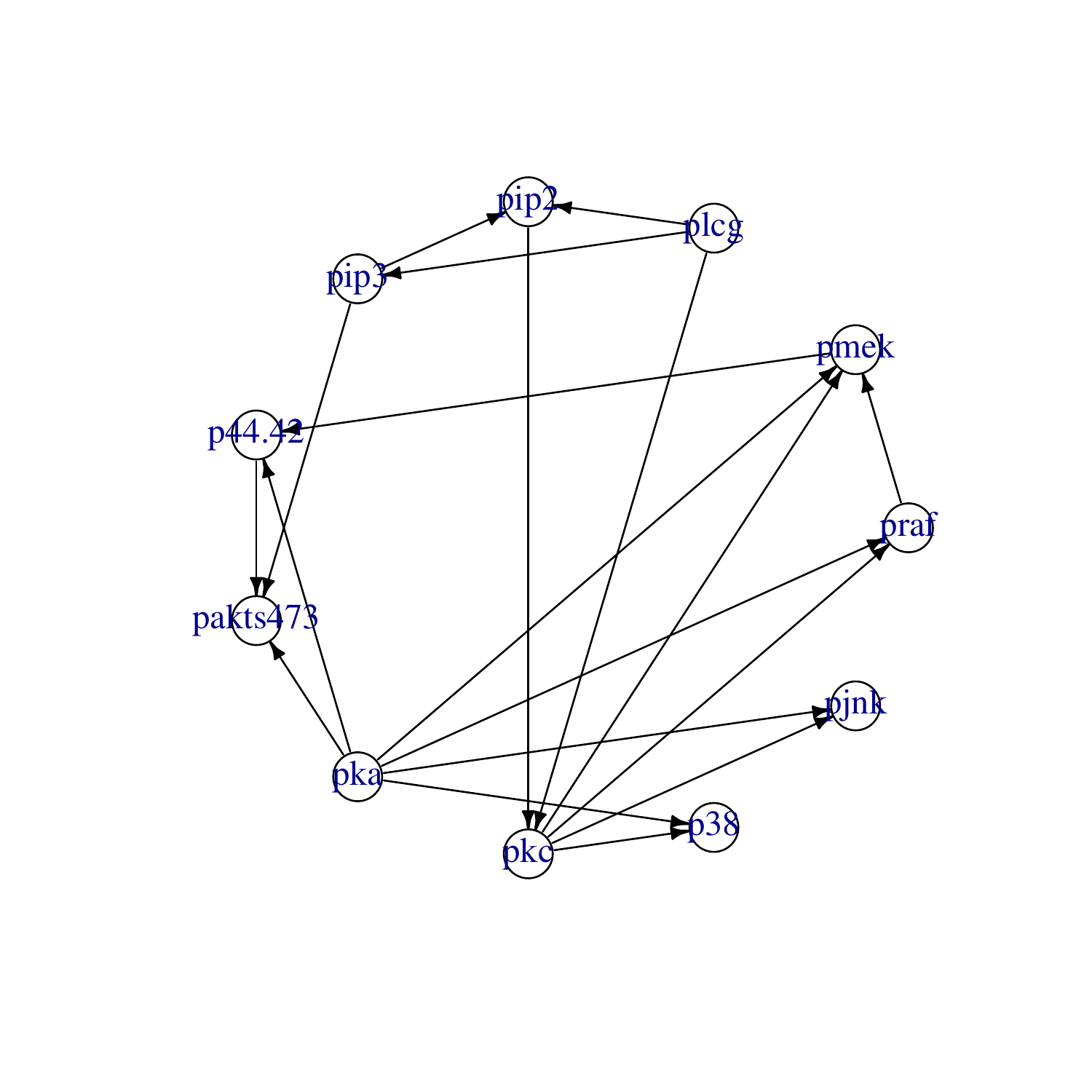}}
    \hspace{0.1\textwidth}
    \subfigure[]{\includegraphics[height=6cm,width=6cm,angle=0]{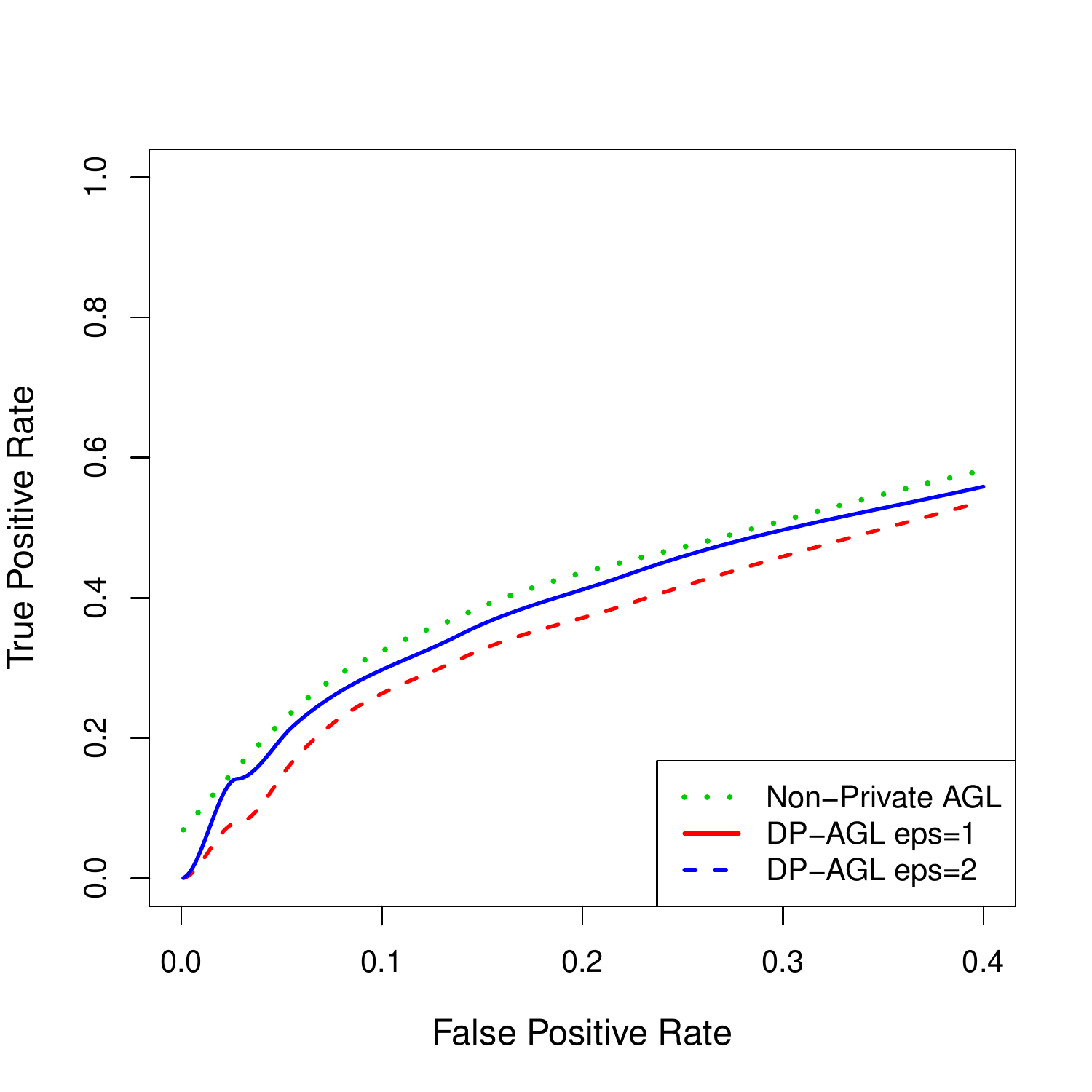}}
    \caption{(a)Directed acyclic graph from cell-signaling data, from Sachs et al. \cite{sachs2005causal}. 
    (b)ROC curves for varying levels of privacy parameter $\epsilon$, compared with the non-private algorithm.}
    \label{fig:3}
\end{figure*}

\begin{figure*}[!h]
    \centering
    \subfigure[]{\includegraphics[height=4.5cm,width=4.5cm,angle=0,trim=40 40 40 40,clip]{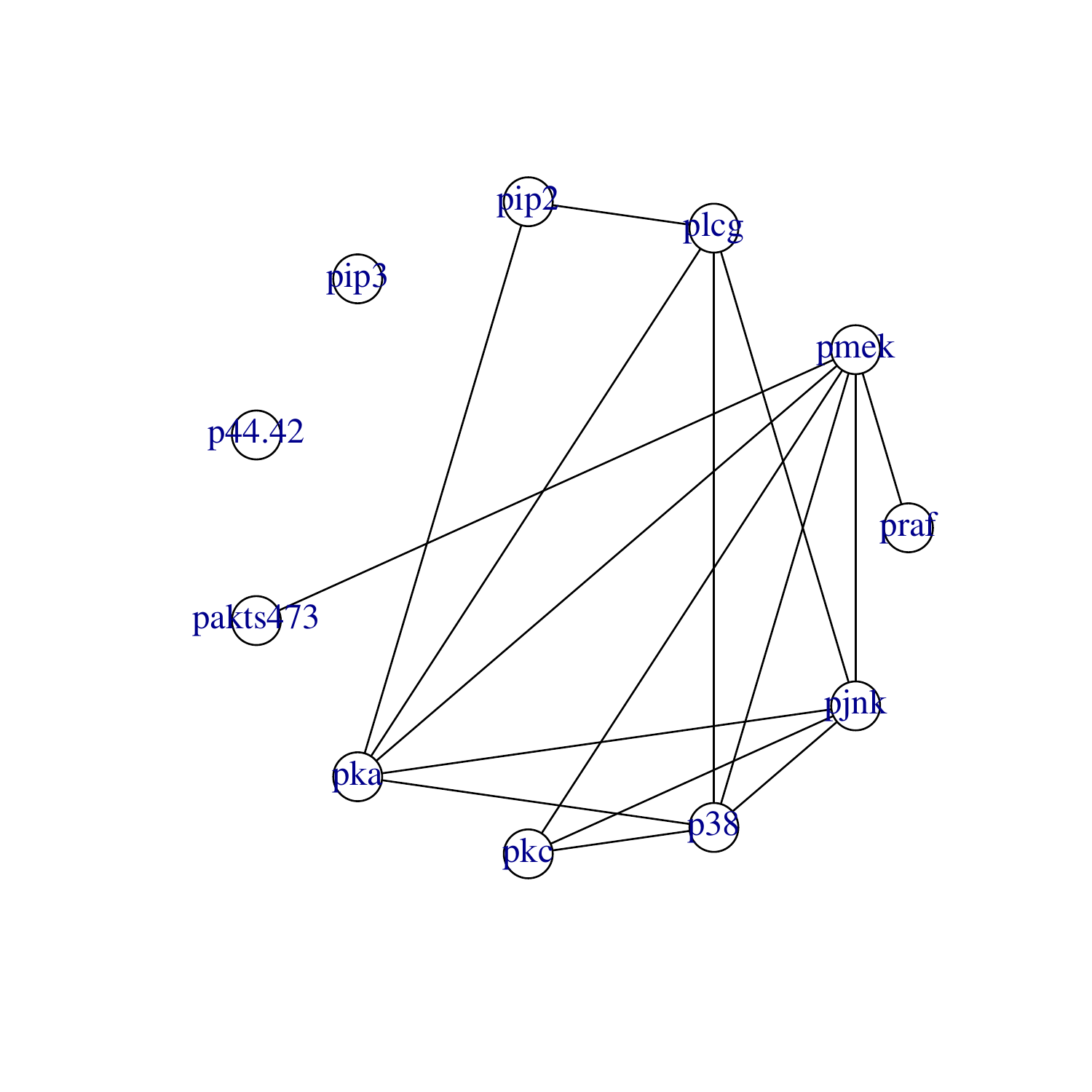}}
    \hspace{0.01\textwidth}
    \subfigure[]{\includegraphics[height=4.5cm,width=4.5cm,angle=0,trim=40 40 40 40,clip]{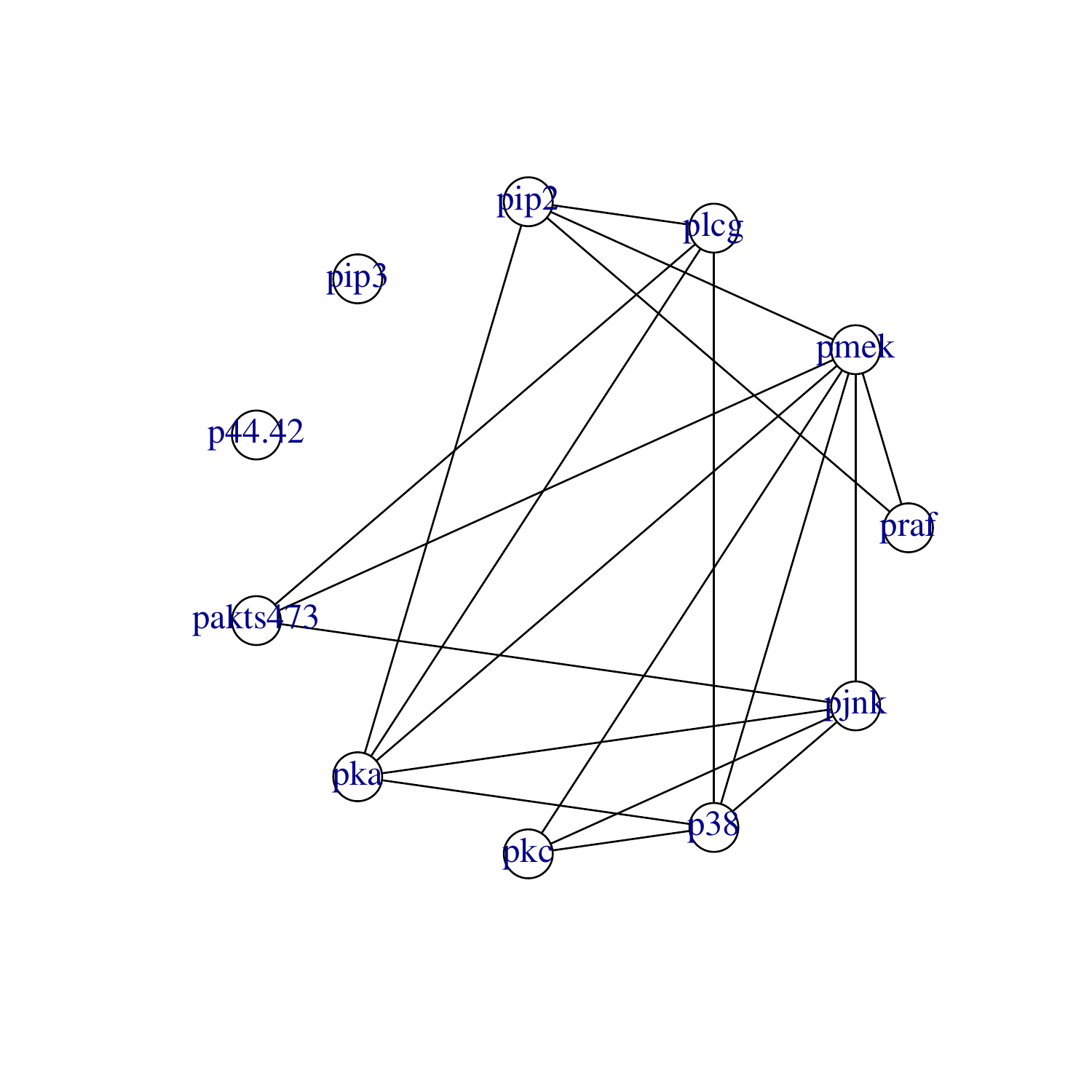}}
    \hspace{0.01\textwidth}    
    \subfigure[]{\includegraphics[height=4.5cm,width=4.5cm,angle=0,trim=40 40 40 40,clip]{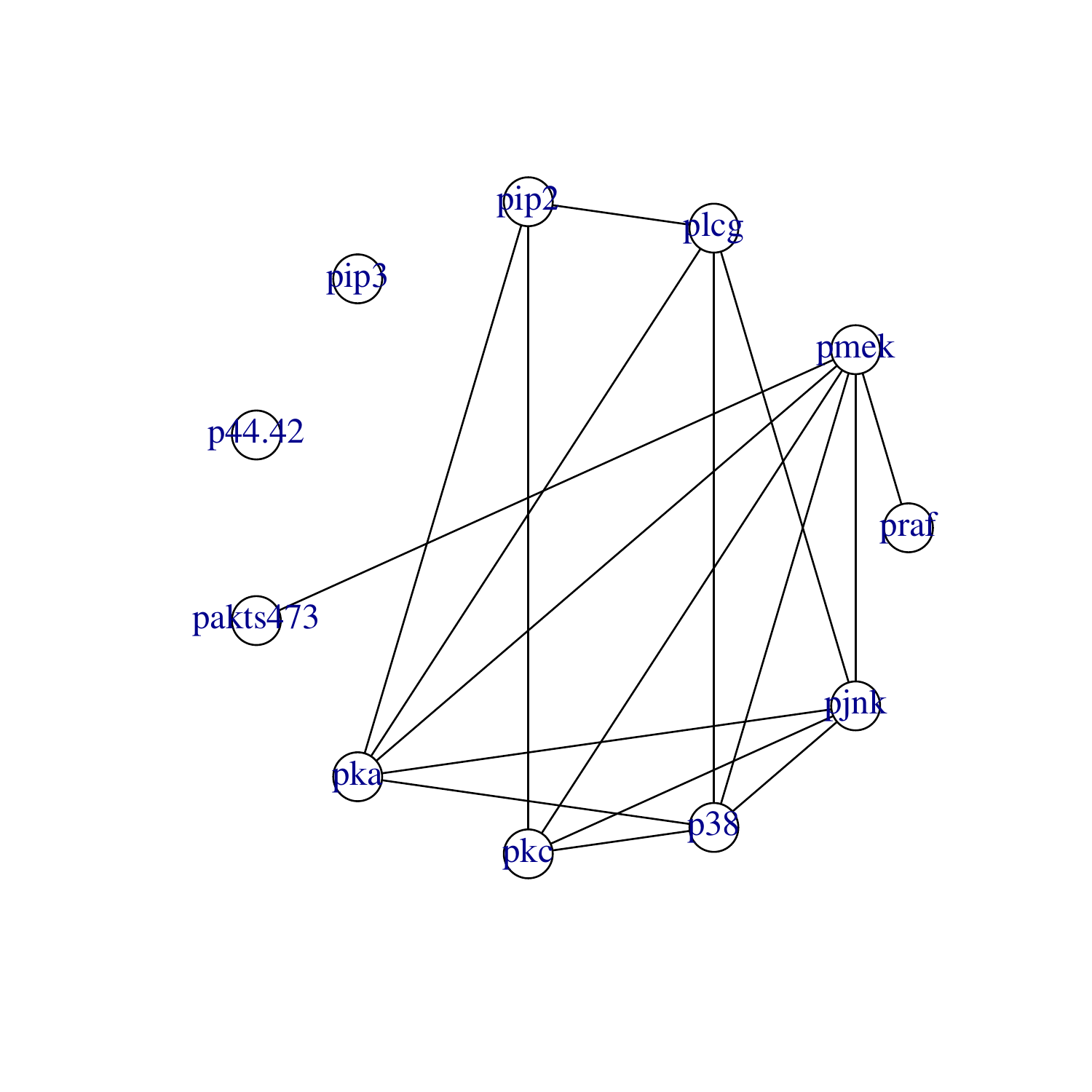}}
\\
    \subfigure[]{\includegraphics[height=4.5cm,width=4.5cm,angle=0,trim=40 40 40 40,clip]{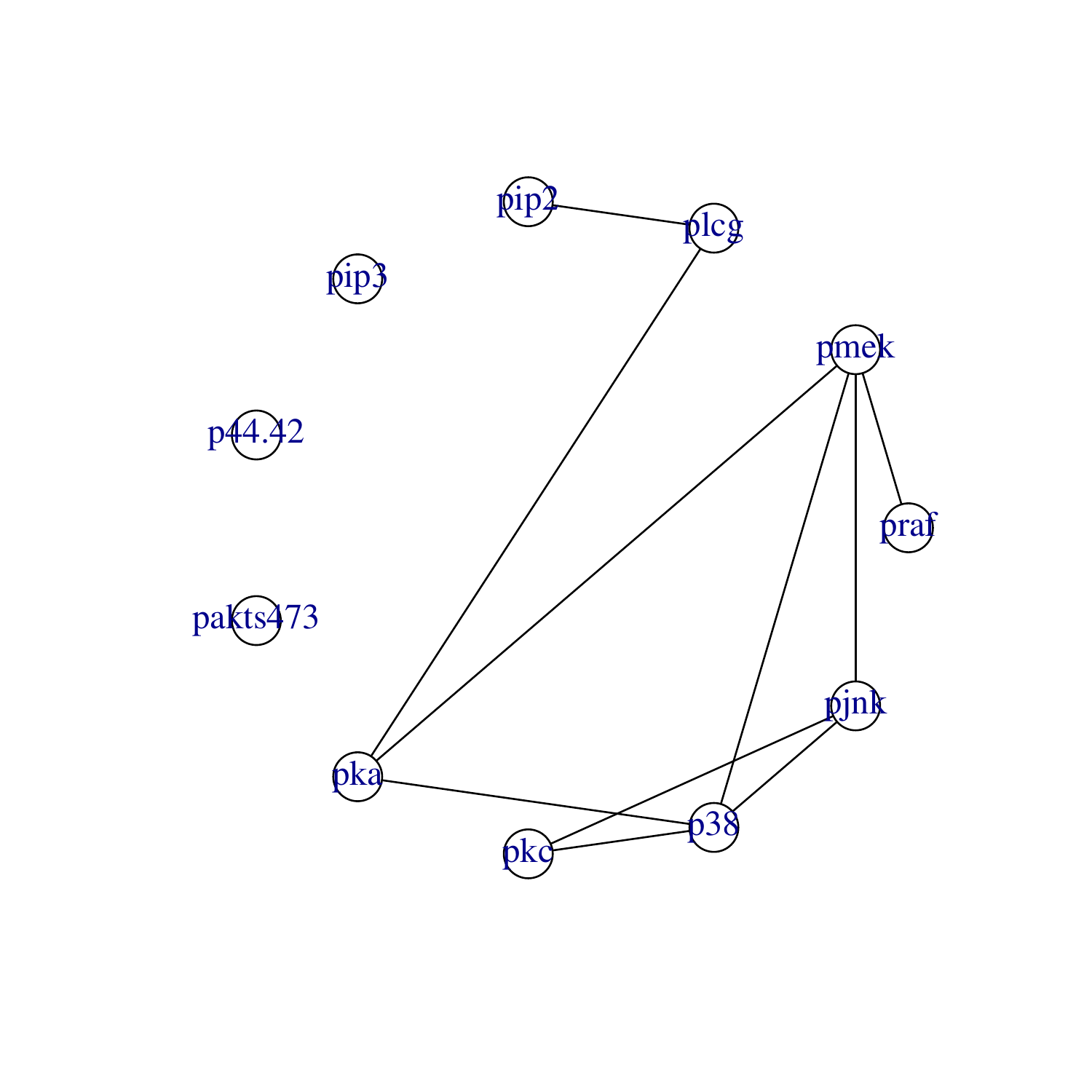}}
    \hspace{0.01\textwidth}
    \subfigure[]{\includegraphics[height=4.5cm,width=4.5cm,angle=0,trim=40 40 40 40,clip]{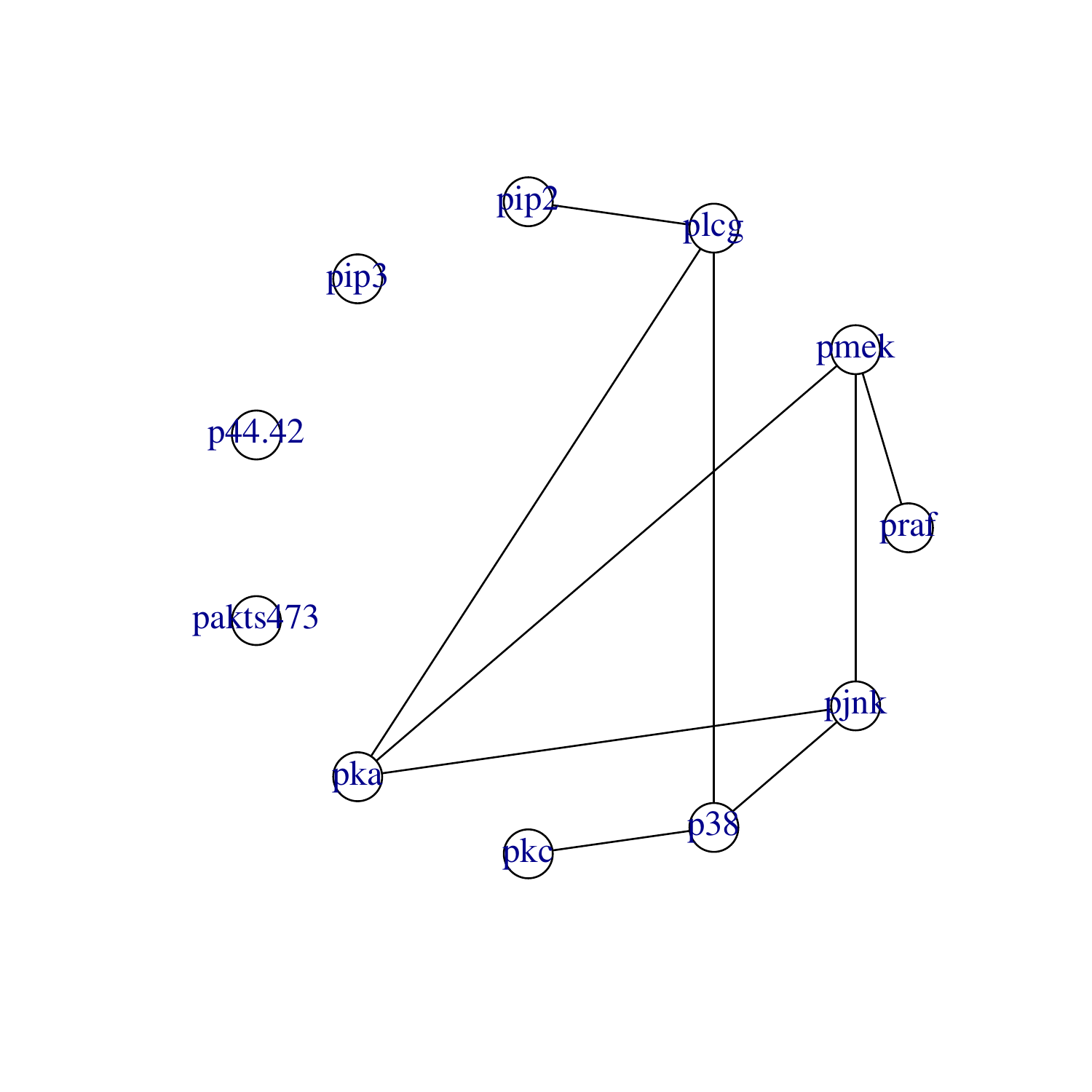}}
    \hspace{0.01\textwidth}    
    \subfigure[]{\includegraphics[height=4.5cm,width=4.5cm,angle=0,trim=40 40 40 40,clip]{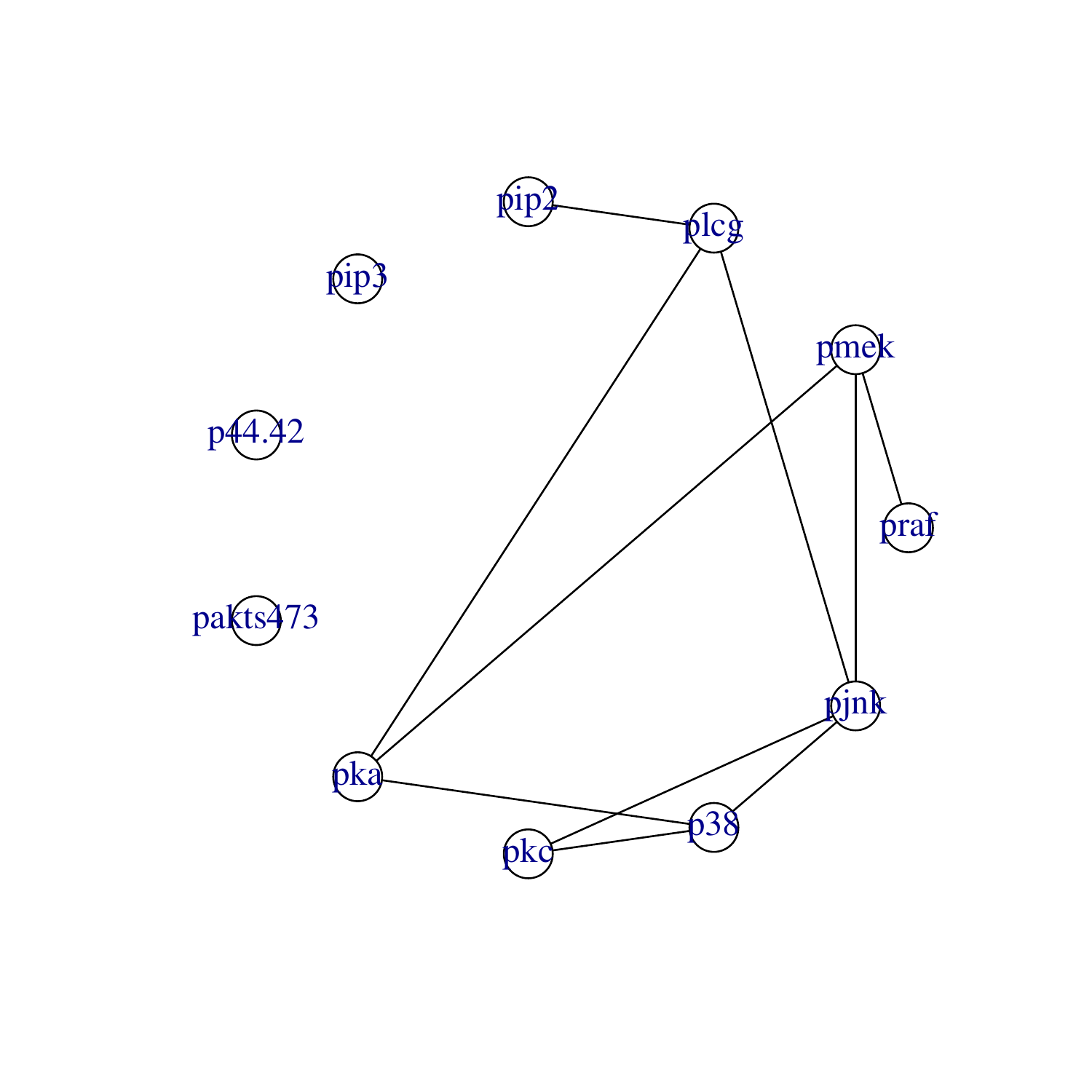}}
    \caption{Undirected graphs from DP-ALG with different values of tuning parameter $\lambda$ and privacy parameter $\epsilon$, compared with undirected graphs from the non-private algorithm. The tuning parameter $\lambda$ corresponding to the subgraphs in the horizontal lines is 0.0015, 0.002, respectively. The first column of subgraphs corresponds to the non-private algorithm and  $\epsilon$ of the subgraphs corresponding to the last two columns is 1, 2, respectively.}
    \label{fig:4}
\end{figure*}

\section{Conclusion} \label{conclusion} 

In this paper, we have focused on developing differentially private algorithms for solving the ridge estimation of the precision matrix and the graphical lasso problem. We first present a differentially private ridge estimation for the precision matrix. Then based on this algorithm and ADMM algorithm, a differentially private algorithm for the graphical lasso problem has been proposed. We compare our differentially private algorithms with the non-private algorithms through different privacy budget and sample size under simulation and real data, which show that our developed differentially private algorithms provide fine utility.

\end{document}